\documentclass{llncs}
\usepackage[leqno]{amsmath}
\usepackage{amsfonts,latexsym,amssymb}
\usepackage{t1enc}
\usepackage{mathptmx}
\usepackage[utf8]{inputenc}
\usepackage{hyperref}
\usepackage{comment}
\usepackage{graphicx}
\usepackage{pgfplots}
\pgfplotsset{compat=1.14}

\parskip=\medskipamount
\allowdisplaybreaks

\numberwithin{equation}{section}

\newcommand{\N}{\ensuremath{\mathcal N}}
\newcommand{\Ke}{\ensuremath{\widehat{~\mathcal K}}} 
\newcommand{\z}{\emptyset}

\newcommand{\Iff}{\Longleftrightarrow}
\newcommand{\tiff}{if and only if \ }

\newcommand{\klam}[1]{\ensuremath{\langle #1 \rangle}}
\newcommand{\set}[1]{\ensuremath{\{#1\}}}
\newcommand{\tand}{\text{ and }}
\newcommand{\tor}{\text{ or }}
\newcommand{\card}[1]{\ensuremath{\lvert #1 \rvert}}

\newcommand{\df}{\ensuremath{:=}}

\newcommand{\bel}{\ensuremath{\operatorname{{bel}}}}
\newcommand{\pl}{\ensuremath{\operatorname{{pl}}}}
\newcommand{\ind}{\ensuremath{\operatorname{{ind}}}}
\newcommand{\upp}{\ensuremath{\operatorname{{upp}}}}
\newcommand{\low}{\ensuremath{\operatorname{{low}}}}
\newcommand{\obs}{\ensuremath{\operatorname{{obs}}}}
\newcommand{\cp}{\operatorname{cp}}
\newcommand{\At}{\ensuremath{\operatorname{{At}}}}
\newcommand{\at}{\ensuremath{\operatorname{{noa}}}}
\newcommand{\pp}{\ensuremath{\operatorname{{pp}}}}

\title{Approximation by filter functions}

\author{Ivo D\"{u}ntsch%
\thanks{The ordering of authors is alphabetical and equal authorship is implied.}
\thanks{Ivo D{\"u}ntsch gratefully acknowledges support by Fujiang Normal Univeristy, the Natural Sciences and Engineering Research Council of Canada Discovery Grant 250153, and by the Bulgarian National Fund of Science, contract
DN02/15/19.12.2016.}
\and G{\"u}nther Gediga$^\star$%
\and Hui Wang$^\star$
}
\institute{School of Maths. and Computer Science, Fujian Normal University,  Fuzhou, Fujian, China, and
Brock University, St. Catharines,  ON, L2S 3A1, Canada, \\ \href{mailto:ivo@duentsch.net}{ivo@duentsch.net}
\and
Institut f{\"u}r Evaluation  und Marktanalysen,  Brinkstr.~19, 49143 Jeggen, Germany, \href{mailto:gediga@eval-institut.de}{gediga@eval-institut.de}
\and School of Maths. and Computer Science, Fujian Normal University,  Fuzhou, Fujian, China, and
School of Computing and Mathematics,
University of Ulster, Newtownabbey, N.Ireland, \\ \href{mailto:H.Wang@ulster.ac.uk}{H.Wang@ulster.ac.uk}
}

\begin{document}

\maketitle

\begin{abstract}
\noindent In this exploratory article, we draw attention to the common formal ground among various estimators such as the belief functions of evidence theory and their relatives, approximation quality of rough set theory, and contextual probability. The unifying concept will be a general filter function composed of a basic probability and a weighting which varies according to the problem at hand. To compare the various filter functions we conclude with a simulation study with an example from the area of item response theory.

\keywords{Filter functions. Belief functions. Approximation quality. Contextual probability.}
\end{abstract}

\section{Introduction}\label{sec:intro}

In order to classify a data point $x \in Q$ about which we have no precise knowledge, one may take into  account information that is available in a neighbourhood of $x$ and use this to classify $x$.  Neighbourhoods can be defined in various ways; prominent examples are by distance functions in a numerical context or as equivalence or similarity classes with respect to a chosen relation in a nominal context \cite{paw91}.

The original rough set concept of neighbourhood of a point $x$ is a class of an equivalence relation which contains $x$. This was generalized to consider the relationship of subsets of $Q$ with $R(x)$, where $R$ is a binary relation on $Q$ and $R(x) = \set{y \in Q: xRy}$. From each of these neighbourhood concepts lower and upper approximations can be derived, and we invite the reader to consult \cite{sv95} for an introduction to such generalization.

Even if we have decided in principle which type of neighbourhood of $E \subseteq Q$ should be considered, it is often still not clear which neighbourhood should be used. For example, one crucial issue in the $k$ -- nearest neighbour method is the choice of $k$. In other words, decisions have to be made which sets we allow to be neighbourhoods of a point or a set, and this is where filter functions come in useful.

The Oxford English Dictionary gives various definitions of \emph{filter}, among others, \cite{oed_filter}:

\begin{itemize}
\item A porous device for removing impurities or solid particles from a liquid or gas passed through it.
\item A device for suppressing electrical or sound waves of frequencies not required.
\item  \emph{Computing} A function used to alter the overall appearance of an image in a specific manner.
\item \emph{Computing} A piece of software that processes data before passing it to another application, for example to reformat characters or to remove unwanted types of material.
\end{itemize}
A filter function may be considered as a rule that tells us which sets are selected to serve as an approximation (or description) of a subset $E$ of the universe $Q$ , and how these  ``neighbourhoods'' will be weighted.


Throughout, $Q$ denotes a finite nonempty set with $\card{Q} = n$, and \N\ is a family of subsets of $Q$.

At times, we will suppose that \N\ is a -- not necessarily proper --  Boolean subalgebra of $2^Q$ with atom set $\At(\N) = \set{A_1, \ldots, A_k}$. In this case, if $Y \in \N$, we define $\at(Y)$ as the number of atoms of \N\ contained in $Y$.

A \emph{probability measure} on a Boolean subalgebra $\N$ of $2^Q$ is an additive function $p$ on $\N$, i.e. if $Q_1, \ldots, Q_k \in \N$, and the $Q_i$ are pairwise disjoint, then $p(\bigcup\set{Q_i: 1 \leq i \leq k} = \sum\set{p(Q_i): 1 \leq i \leq k}$; we require furthermore that $p(Q) = 1$. This is the standard definition of measure theory.

The \emph{sampling probability} on \N\ is defined by
\begin{gather}\label{rp}
p_\N(Y) \df
\begin{cases}
\sum\set{\frac{\card{A_i}}{n}: A_i \subseteq Y}, &\text{if } Y \neq \z, \\
\z, &\text{otherwise.}
\end{cases}
\end{gather}
This assignment is based on the principle of indifference and assumes ignorance about the distribution within the atoms of \N.

A generalization of probability measures are \emph{mass functions} or \emph{basic probabilities} \cite{sha76}, or \emph{basic belief functions} \cite{sk_tbm}: A \emph{mass function on \N} is a function $m: \N \to [0,1]$ such that $\sum\set{m(Y): Y \in \N} = 1$. A \emph{focal element} is a set $Y \in \N$ with $m(Y) \neq \z$.  Owing to the finiteness of $Q$, the restriction to the upper bound $1$ for $m(Y)$ is one of convenience which may be obtained by appropriate weighting. Unlike the Dempster--Shafer model, we assume an open world situation, and do not require that $m(\z) = 0$; here, we follow \cite[Section 4.8]{smets88a}.

If $p$ is a probability measure on $\N$, then the function $m_p: \N \to [0,1]$ defined by
\begin{gather}\label{mp}
m_p(Y) \df
\begin{cases}
p(Y), &\text{if } Y \in \At(\N), \\
0, &\text{otherwise.}
\end{cases}
\end{gather}
is a mass function. So, formally, probabilities are special mass functions (often called \emph{Bayesian mass functions}).

\section{Filter functions}

In general, a filter is a function which passes information that is pertinent to the application area,  and reduces (or leaves out) information considered to be irrelevant. This concept of a filter originates with signal processing, but the same idea may be applied to elements of weighted structures. There is no relation to the filter concept in lattice theory.

We consider filter functions $F: 2^Q \to [0,1]$ of the general form
\begin{gather}\label{genF}
F(E) = \sum\set{m(Y) \cdot w(E,Y), Y \in \N}.
\end{gather}
A filter consists of several parts:
\begin{itemize}
\item A set \N\ of neighbourhoods which are often determined by an indicator function and, perhaps, other parameters. In such a way, the pool \N\ of possible neighbourhoods is adjusted to the needs of the problem under consideration. How the initial \N\ is chosen is a topic for further research.
\item A weighting function $w: 2^Q \times \N \to [0,1]$ which re--scales the weights of the neighbourhoods in such a way that desired properties such as the value of an upper bound or the sum of the re--scaled values are guaranteed. In most cases, the values of $w$ will be in $\set{0,1}$.
\end{itemize}


If $E \subseteq Q$ is an event (or a piece of evidence), and $Y \in \N$, it is reasonable to suppose that $Y$ should not be considered a neighbourhood of $E$, if $E \cap Y = \z$. On the other hand, any $Y$ which contains $E$ should be considered a neighbourhood of $E$; these are, in some sense, ``boundary'' situations.

In this spirit, we define our main indicator functions by
\begin{xalignat*}{2}
\ind^u(X, Y) = 1 &\Iff \text{ if } X \cap Y \neq \z, &&\text{Upper indicator} \\
\ind^l(X, Y) = 1 &\Iff Y \subseteq X, &&\text{Lower indicator}. \\
\intertext{Other indicators we use are}
\ind^z(Y) = 1 &\Iff \ind^u(Y, Y) = 1  \Iff Y \neq \z, \\
\ind^{sub}(X,Y) = 1 &\Iff X \subseteq Y, &&\text{Subset indicator}, \\
\ind^{eq}(X,Y) = 1 &\Iff \ind^{sub}(X,Y) \cdot \ind^{sub}(Y,X) = 1 \Iff X = Y &&\text{Equality indicator}.\\
\end{xalignat*}
We suppose, as is customary, that an indicator function takes values in $\set{0,1}$. Now we define the \emph{upper} and the \emph{lower filter}:
\begin{xalignat}{2}
F^u_m(E) &\df \sum \set{m(Y) \cdot \ind^u(E, Y): Y \in \N}, &&\text{Upper filter} \label{upF} \\
F^l_m(E) &\df \sum\set{m(Y) \cdot \ind^l(E, Y): Y \in \N}, &&\text{Lower filter}. \label{lowF}
\end{xalignat}

Lower and upper filters as defined above are not the only one, which select a neighbourhood of some evidence  $E$;  they are, as we shall see, maximal filters of their type: For the upper filter and $E \neq \z$, a set $Y \in \N$ is a neighbourhood of $E$, if they have at least one element in common. A simple way to sharpen this is the demand that they have at least $k \geq 1$ elements in common. If $E$ has exactly one element, then the situation is unchanged, but if $E$ consists of more than one element, the number of neighbourhood sets will be reduced. These considerations lead us to \emph{upper} and \emph{lower $k$ -- filters} $(1\leq k\leq |Q|)$ by first defining the indicators
\begin{align}
\ind^{u,k}(X,Y) = 1 &\Iff   \card{X \cap Y} \geq k , \label{ukI}\\
\ind^{l,k}(X,Y) = 1 &\Iff   Y \subseteq X  \tand \card{Y} \geq k. \label{lkI}
\end{align}
A similar parametrization may be used to demand that a neighbourhood should cover more than $s$\% of the event. So, we define the indicator functions
\begin{align}
\ind^{u,s}(X,Y) = 1 &\Iff X = Y \tor \card{X \cap Y} \gneq s \cdot \card{X}, \label{usI}\\
\ind^{l,s}(X,Y) = 1 &\Iff X = Y \tor Y \subsetneq X  \tand \card{Y} \gneq s \cdot \card{X}. \label{lsI}
\end{align}
The boundary values of the parameterized indicators are easily seen to be
{\small
\begin{align*}
\ind^{u,k=1}(X,Y) = \ind^{u,s=0}(X,Y) &= \ind^u(X,Y), & \ind^{u,k=|Q|}(X,Y) = \ind^{u,s=1}(X,Y) &= \ind^{sub}(X,Y)  \\
\ind^{l,k=1}(X,Y) = \ind^{l,s=0}(X,Y) &= \ind^{l}(X,Y), & \ind^{l,k=|Q|}(X,Y) = \ind^{l,s=1}(X,Y) &= \ind^{eq}(X,Y).
\end{align*}
}
The respectively weighted upper and lower filter are now defined by
\begin{xalignat}{2}
F^{u,s}_m(E) &\df \sum_{Y \in \N} m(Y) \cdot \ind^{u,k}(E, Y), \label{upkF} \\
F^{l,s}_m(E) &\df \sum_{Y \in \N} m(Y) \cdot \ind^{l,k}(E, Y), \label{lowkF} \\
F^{u,s}_m(E) &\df \sum_{Y \in \N} m(Y) \cdot \ind^{u,s}(E, Y), \label{upsF} \\
F^{l,s}_m(E) &\df \sum_{Y \in \N} m(Y) \cdot \ind^{l,s}(E, Y). \label{lowsF}
\end{xalignat}
The parameterized filters are antitone with respect to $s$:

\begin{theorem}
Let $s,t \in [0,1]$, and $s \leq t$. Then, $F^{l,t}_m(E) \leq F^{l,s}_m(E)$ and $F^{u,t}_m(E) \leq F^{u,s}_m(E)$.
\end{theorem}
\begin{proof} We show the claim only for the lower filter, as the remaining claim is proved similarly. First, consider
\begin{align*}
F^{l,t}_m(E) \leq F^{l,s}_m(E) &\Iff F^{l,s}_m(E) - F^{l,t}_m(E) \geq 0, \\
&\Iff \sum_{Y \in \N} m(Y) \cdot \ind^{l,t}(E, Y) - \sum_{Y \in \N} m(Y) \cdot \ind^{l,s}(E, Y) \geq 0, \\
&\Iff \sum_{Y \in \N} m(Y) \cdot ( \ind^{l,t}(E, Y) -  \ind^{l,s}(E, Y)) \geq 0.
\end{align*}
Since $s \leq t$, we have $\card{Y} \gneq t \cdot \card{X}$ implies $\card{Y} \gneq s \cdot \card{X}$, and therefore, $\ind^{l,t}(E,Y) = 1$ implies $\ind^{l.s}(E,Y) = 1$. It follows that $\ind^{l,s}(E,Y) \geq \ind^{l,t}(E,Y)$, i.e. $\ind^{l,s}(E,Y) - \ind^{l,t}(E,Y) \geq 0$. Since $m(Y) \geq 0$, we conclude $F^{l,t}_m(E) \leq F^{l,s}_m(E)$.
\end{proof}
The same proof shows that the parameterized filters are antitone as well.

\section{Approximation and estimation}\label{sec:est}

In this section we show how commonly used belief and approximation measures fit into the scheme of filter functions as proposed in \eqref{genF}. For an overview of different interpretations of ``belief'' we refer the reader to \cite{hf92}.

\subsection{Evidence measures}

Evidence theory has been widely studied as an alternative to classical probability theory, see the source book  edited by Yager \& Liu \cite{yl08}. For a thoughtful discussion of belief and probability we invite the reader to consult \cite{fh91} and \cite{hf92}, where, among others, it was shown  that ``a key part of the important Dempster-Shafer theory of evidence is firmly rooted in classical probability theory''.

In evidence theory and related fields, two functions are obtained from a mass function $m: \N \to [0,1]$:
\begin{xalignat}{2}
\bel_m(E) &\df \sum_{Y \in \N, Y \subseteq E} m(Y), &&\text{degree of belief}, \label{def:bel} \\
\pl_m(E) &\df \sum_{Y \in \N, Y \cap E \neq \z} m(Y), && \text{degree of plausibility.} \label{def:pl}
\end{xalignat}
These concepts were introduced by Dempster \cite{dempster1967}, who called them, respectively, \emph{lower} and \emph{upper probability}. A belief function assigns the total amount of belief supporting $E$ without supporting $Q \setminus E$, and $\pl_m(E)$ quantifies the maximal amount of belief that might support $E$ \cite{smets88}. It is straightforward to show that $\pl_m(E) = \bel_m(Q) - \bel_m(Q \setminus E)$.

Conversely, every mass function can be obtained from a function $\bel$ which satisfies certain conditions, see e.g. \cite[Chapter 2]{sha76}.

Belief and plausibility are easily related to the upper and lower filter function as follows:
\begin{align*}
\bel_m(E) &= \sum\set{m(Y): Y \subseteq E, Y \in \N} = \sum\set{m(Y) \cdot \ind^l(E, Y): Y \in \N} = F^l_m(E), \\
\pl_m(E) &= \sum\set{m(Y): E \cap Y \neq \z, Y \in \N} = \sum\set{m(Y) \cdot \ind^u(E, Y): Y \in \N} = F^u_m(E).
\end{align*}

\subsection{Rough set approximation quality}

Suppose that $X \subseteq Q$, and that $\N$ is a Boolean algebra with atoms $A_1, \ldots, A_k$. Then, $\At(\N)$ can be considered the partition of $Q$ obtained from some equivalence relation $\theta$ on $Q$; in other words, we work with a rough set approximation space $\klam{Q,\theta}$.  In rough set theory \cite{paw91}, the \emph{upper approximation of $X$} is the set $\upp(X)\df \bigcup\set{A_i: A_i \cap X \neq \z}$ and the \emph{lower approximation of $X$} is the set $\low(X) \df \bigcup\set{A_i: A_i \subseteq X}$. These approximations lead to two statistics relative to \N:
\begin{align}
\mu^{\N *}(E) &= \frac{\card{\upp(E)}}{n}, \\
\mu^\N_*(E) &= \frac{\card{\low(E)}}{n}.
\end{align}

Inspection of the indices used in ``classical rough set theory''  such as $\alpha, \gamma$, rough membership, other element counting etc. shows that these indices are valid only in case we assume the principle of indifference: Assuming no knowledge of the distribution within the equivalence classes, we let $p$ be the sampling probability measure on \N\ as defined in \eqref{rp}. There may be other assumptions within the frame of lower and upper set approximations, which consequently lead to other evaluation schemes. The principle of indifference is widely used in rough set theory -- explicitly or implicitly. For example, the general rough membership function defined in \cite[Definition 4.3.]{mani17} is a special filter in our terminology for which the principle of indifference is a hidden assumption; otherwise the estimator of this index is biased and unsuitable for applications. In \cite{mani17} only point estimators of indices or membership functions are addressed - but this is not the whole story: The reliability of the indices needs to be discussed as well. Assuming the principle of indifference, we are able to compute confidence intervals such as the reliability  of the general rough membership function or other filters, as we demonstrate in the present work.

Using the mass function $m$ determined by $p$ as defined in \eqref{mp} we can describe $\mu^{\N *}(E)$ and $\mu^\N_*(E)$ in terms of upper and lower filter:
\begin{align*}
\mu^{\N *}(E) &= \sum\set{\frac{\card{A_i}}{n}: E \cap A_i \neq \z}, \\
&= \sum\set{m(Y): E \cap Y \neq \z,  Y \in \N}, \\
&= \sum\set{m(Y) \cdot \ind^u(E \cap Y), Y \in \N}, \\
& = F^u_m(E), \\
\mu^\N_*(E) &= \sum\set{\frac{\card{A_i}}{n}: A_i \subseteq E)}, \\
&=  \sum\set{m(Y)\cdot \ind^l(E \cap Y), Y \in \N}, \\
& = F^l_m(E).
\end{align*}
This shows the close connection of rough set approximation to the estimators of evidence theory, observed first by Skowron \cite{sko90}.

 The \emph{approximation quality} is the function
\begin{align}
\gamma(E) &\df \frac{\card{\low(E)}}{n} + \frac{\card{\low(Q \setminus E)}}{n}. \label{gamma}.
\end{align}
$\gamma(E)$ is the relative frequency of all elements of $Q$ which are correctly
classified under the granulation of information by $\N$  with respect to being an element
of $E$ or not. In terms of filter functions, this becomes
\begin{align}
\gamma(E) &= F^l_m(E) + F^l_m(Q \setminus E).
\end{align}
%

\subsection{Pignistic probability}

According to Smets \cite{smets88}, decision making under uncertainty can (and should) be done in two steps. On a \emph{credal level}, an assignment of beliefs is made to pieces of evidence. In order to be coherent on a \emph{pignistic level} (decision level), the uncertainties quantified by the belief function must be turned into a probability measure. In such a way, the two levels of handling uncertainty and decision making are clearly separated unlike, as Smets claims, in Bayesian reasoning.

 A \emph{pignistic probability distribution} (with respect to the mass function $m$ and the Boolean algebra \N) \cite[Section 3]{sk_tbm} is a function $\pp: \N \to [0,1]$ which is defined by
\begin{gather}\label{def:pig}
\pp_m(E) \df \sum\set{ m(Y) \cdot \frac{\card{E \cap Y}}{\at(Y)}: Y \in \N^+}
\end{gather}
If $E$ is an atom of \N, we obtain
\begin{gather}
\pp_m(E) = \sum\set{ m(Y) \cdot \frac{\card{E}}{\at(Y)}: E \subseteq Y \in \N}.
\end{gather}
Note that $E \subseteq Y$ implies that $Y \neq \z$. It was shown in \cite{smets88} that $\pp$ is indeed a probability measure, if $\N = 2^Q$. Setting
\begin{gather*}
w(E,Y) \df
\begin{cases}
\frac{\card{E \cap Y}}{\at(Y)} &\text{if } Y \neq \z, \\
0, &\text{otherwise},
\end{cases}
\end{gather*}
we see that $\pp(E) = \sum\set{m(Y) \cdot w(E,Y), Y \in \N}$ as in \eqref{genF}.

\subsection{Contextual probability}

Another two step procedure to reason under uncertainty, called \emph{contextual probability}  was first proposed in \cite{hui03}, and subsequently developed in \cite{wdgg_noknn}. It is a secondary probability, which is defined in terms of a basic (primary) function; it can be used to estimate the primary probability from a data sample through a process called {\em neighbourhood counting}; for details see \cite{wm08}.

Given a mass function $m$ over $2^Q$, we first define a weight function by
\begin{gather*}
w(E,Y) \df
\begin{cases}
\frac{\card{E \cap Y}}{\card{Y}} &\text{if } Y \neq \z, \\
0, &\text{otherwise}.
\end{cases}
\end{gather*}
The \emph{contextual probability} is the function $\cp^m: 2^{Q}\rightarrow [0,1]$ defined by
\begin{gather}\label{def:cpm}
\cp^m(E))=\sum\set{ m(Y)  \cdot w(E,Y):  Y \in \N},
\end{gather}
Wang \cite{hui03} showed that $\cp^m$ is a probability distribution if $\N = 2^Q$.

This definition of contextual probability was found problematic when trying to find a simple relationship between the primary probability and the secondary probability, so the definition was refined in \cite{wd05}, and extended in \cite{wm08}. The work on estimating contextual probability from data sample has spawned a series of papers exploring the various forms of neighbourhood counting for multivariate data, sequences, trees, and graphs. We give a somewhat simplified version of the revised definition, and also extend its range over $2^Q$.

Suppose that $p$ is a probability measure on \N, and let  $K \df \sum\set{p(Y) \cdot \card{Y}: Y \in \N}$ be a normalization factor. The \emph{contextual probability with respect to $p$},  is defined by
\begin{gather}\label{def:cpp}
\cp^p(E) \df \sum\set{p(Y) \cdot \frac{ \card{E \cap Y}}{K}, Y \in \N}.
\end{gather}
Setting $w(E,Y) \df \frac{ \card{E \cap Y}}{K}$ and using the mass function $m_p$ of \eqref{mp}, we see that $\cp^p$ is an instance of a general filter function.

\section{Probabilistic knowledge structures}\label{sec:pks}

In this section we apply some of the filter functions defined previously to a situation well known in the context of psychometric aspects of learning, in particular, knowledge structures \cite{fal90,fd11}. Connections of knowledge structures to other concepts including rough sets were exhibited in \cite{dg_ksencon}.

Suppose that $U$ is a set of students, $Q$ is a set of problems, and $S \subseteq U \times Q$ is a binary relation between students and problems, called a \emph{solving relation}; $uSq$ means that student $u$ solves problem $q$. For each $u \in U$, the set $S(u) \df \set{q \in Q: uSq}$ is called the \emph{empirical (observed) solving pattern of $u$}. The set $\set{S(u): u \in U}$ is called an \emph{empirical knowledge structure} (EKS) with respect to $U$ and $Q$, denoted by $\Ke$. With each $X \subseteq Q$ we associate a number $\obs(X) = \card{\set{u \in U: S(u) = X}}$. Thus, $\obs(X)$ is the number of times that $X$ was observed as a student's solving pattern.

A \emph{probabilistic knowledge structure} (PKS) is a tuple $\klam{\N,m}$ where $\N \subseteq 2^Q$, and $m$ is a mass function on \N.  We interpret $m$ as \emph{item--pattern probability} in the sense that
\begin{gather}
m(X) = p(\text{each $x \in X$ is solved, and no problem in $Q \setminus X$ is solved}).
\end{gather}
in other words $m(X)$ is the probability that $X$ is an observed item pattern. $m(\z)$ is the probability that no item in $Q$ is solved, and $m(\set{x})$ is the probability that only $x$ is solved.

Given a PKS, we estimate the probabilities by the relative frequencies of the observed item patterns  by
\begin{gather}
\hat{m}(X) = \hat{p}(\text{each $x \in X$ is solved, and no problem in $Q \setminus X$ is solved}) = \frac{\obs(X) \cdot \card{X}}{n}.
\end{gather}
In this way we not only obtain insight into the probability nature of the mass function and its derivations, but we may use the empirical counterpart of relative frequencies as estimates and as a basis for statistical inference.

Using a PKS as a workhorse, we will explore which interpretation this context offers for different filter functions.
First, consider $F^{l}_m$, which is just the belief function $\bel_m$. Then, according to our interpretation,
\begin{align*}
\bel_m(E) =& \sum\set{m(Y): Y \in \N, Y \subseteq E}, \\
=& p_{\bel_m}((\text{some items in $E$ are solved or no item is solved)}\\
& \text{and no item outside $E$ is solved.})
\end{align*}
Considering a solving path $\z \subseteq \set{x_1} \subseteq \set{x_1, x_2} \subseteq \ldots \subseteq E$, we see that $p_{\bel_m}$ is a cumulative probability function with $p_{\bel_m}(Q) = 1$. A problem which may arise is that the condition ``some item in $E$ is solved or no item is solved'' is not always acceptable. Thus, we may remove the latter condition -- which corresponds to $m(\z) \neq \z$, and define
\begin{align*}
\bel_m^+(E) &= \sum\set{m(Y): Y \in \N^+, Y \subseteq E}, \\
&=  p_{\bel_m^+}(\text{some items in $E$ are solved and no item outside $E$ is solved.})
\end{align*}
$\bel^+_m$ is also a cumulative function, but $\bel^+_m(Q) = 1 - m(\z)$.

Turning to $F^u_m$, we recall that $F^U_m = \pl_m$. Then,
\begin{align*}
\pl_m(E) &= \sum\set{m(Y): Y\cap E\neq\emptyset, Y \in \N}, \\
&= p_{\pl}(\text{at least one problem in $E$ is solved}).
\end{align*}
If $E = \set{x}$, then $p_{\pl}(\set{x})$ is the \emph{item solving probability} of $x$.

To estimate only the states in \N, we let $\ind_{\N}(E) \df 1$ \tiff $E \in \N$, and define
\begin{gather}\label{lowminF}
\bel^{\min}_m(E) \df \ind_\N(E) \cdot F^l_m(E) = \sum \set{m(Y) \cdot \ind_\N(E) \cdot \ind^l(E, Y): Y \in \N}.
\end{gather}
$F^{l, \min}_m$ may be regarded as some sort of minimal lower filter, as only elements of \N\ are allowed to be approximated. Observe that the lower filter $F^l_m$ coincides with $\bel^{\min}_m$ \tiff $\N = 2^Q$.

To parameterize the upper filter $F^u_m(E)$ to use only states in \N\ that contain $E$ we shall consider  $\pl^{\min}_m \df F^{u,1}_m$ as defined in \eqref{upsF} with $s = 1$.

Suppose we have a set of five questions $Q = \set{1,2,3,4,5}$ and \N\ consisting of 12 item patterns, each supplied with a basic probability, as shown in Figure \ref{fig:ks1}.

\begin{figure}[h!tb]
\caption{A weighted knowledge structure}\label{fig:ks1}
\vspace{5mm}
\centering
\includegraphics[width=.9\textwidth, clip]{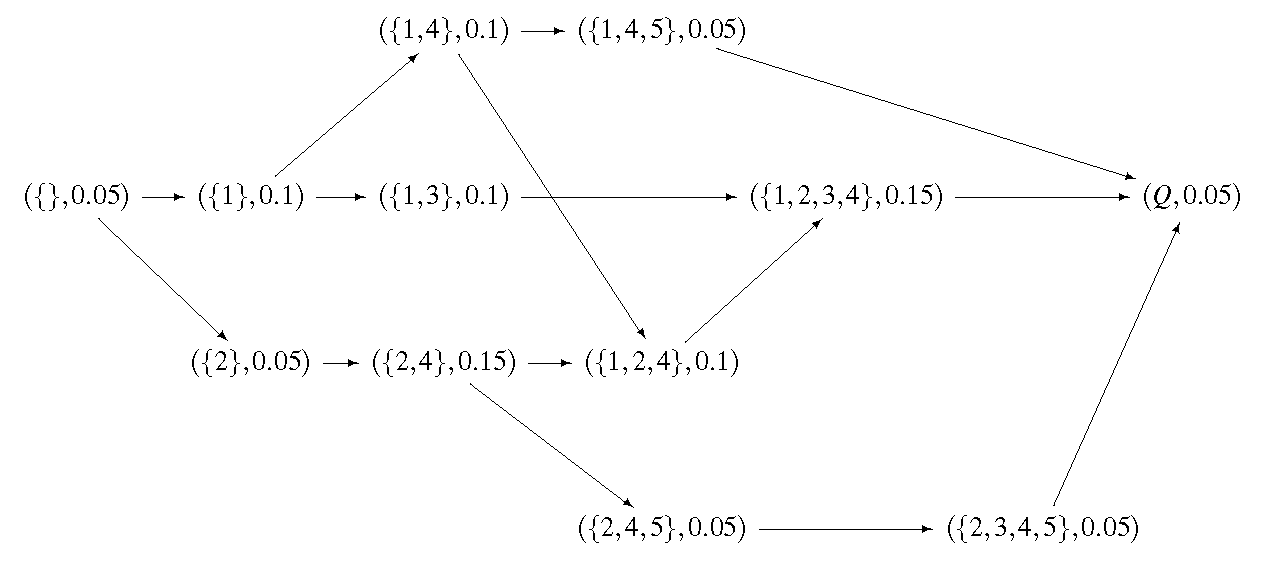}
\end{figure}

Given the PKS in Figure \ref{fig:ks1}, we have performed some empirical experiments to compute sampling distributions of the defined filter procedures. We use a multinomial sampling,  and $N=50, N = 100, \tor N = 1,000$ observations of item patterns. For 10,000 simulations of the sampling process, we computed the sampling distributions of the functions $\bel_m, \pl_m, \cp^m, \bel^{\min}_m, \pl^{\min}_m, \tand pl[k=2]$ for all subsets of $2^{\set{1,2,3,4,5}}$. We have computed the mean, bias, median, upper and lower quartile, and the 2.5\%- and 97.5\%-quantile of the sampling distributions of these functions for each subset $Q$.
\footnote{The  tables and the R-source of the simulation procedure are available for download at \href{www.roughsets.net}{www.roughsets.net}.
}.

\begin{figure}[h!tb]
\caption{Simulation graph}\label{fig:ks2}
\centering
\includegraphics[width=.6\textwidth, clip]{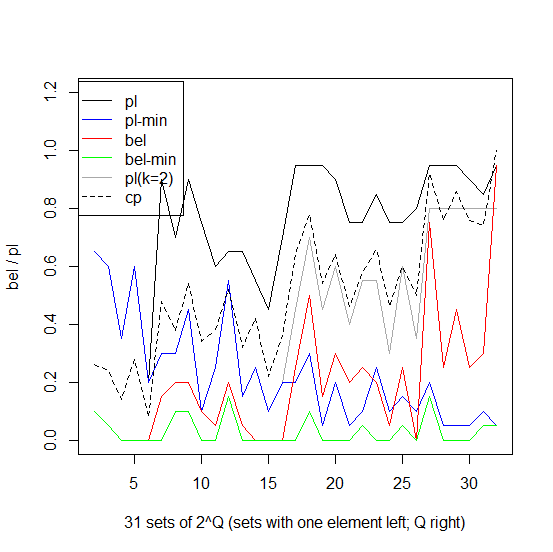}
\end{figure}

\begin{figure}[h!tb]
\newcommand{\tkw}{0.98}
\newcommand{\tkh}{0.6}
\caption{CI and median of $\cp^m$}\label{fig:ks3}
\begin{minipage}[t]{0.48\textwidth}
\vspace{0mm}
\centering
  \begin{tikzpicture}
   \begin{axis}[ymin=0, ymax=1, xmin=0, xmax=33, width=\tkw\textwidth,height=\tkh\textheight,
			title={CI 95\%, $n = 50$)}		
			]
            \addplot table[x=Varname, y=lower]  {datacp50.csv};
			\addplot table[x=Varname, y=median] {datacp50.csv};
			\addplot table[x=Varname, y=upper]  {datacp50.csv};			
	\end{axis}
\end{tikzpicture}
\end{minipage}
\begin{minipage}[t]{0.48\textwidth}
\vspace{0mm}
\centering
\begin{tikzpicture}
   \begin{axis}[ymin=0, ymax=1, xmin=0, xmax=33, width=\tkw\textwidth,height=\tkh\textheight,
			title={CI 95\%, $n = 500$)}			
			]
            \addplot table[x=Varname, y=lower]  {datacp500.csv};
			\addplot table[x=Varname, y=median] {datacp500.csv};
			\addplot table[x=Varname, y=upper]  {datacp500.csv};			
	\end{axis}
\end{tikzpicture}
\end{minipage}
\end{figure}
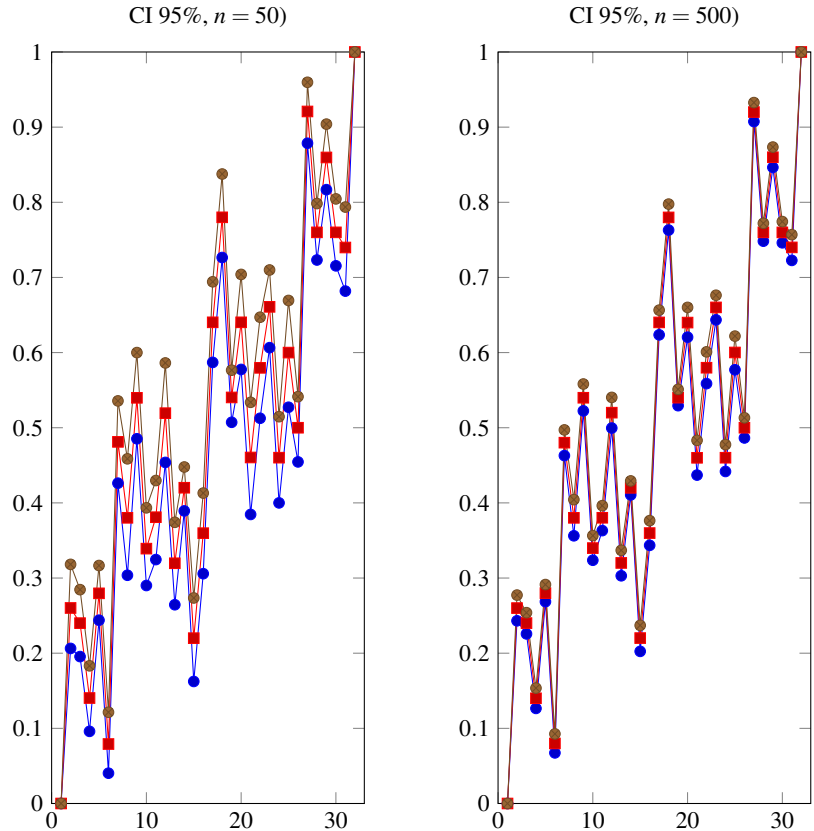

Figure \ref{fig:ks2} shows the mean of the different filter functions on the nonempty subset of $2^Q$.
The left most is the value of $\set{1}$, followed by the values of the sets $\set{2}$, \ldots, $\set{5}$. The sets with two elements follow in lexicographical order, followed by the sets with $3$, $4$, and finally, $5$ elements.

We observe that the values of the functions $\pl_m, pl^{\min}_m$, and $\pl_m^{k=2}$ are equal for sets with one element, and  $\pl^{\min}_m$ and $\pl^{k=2}$ are identical for sets with two elements. The larger the number of elements, the larger the difference of $\pl_m$ and $\pl^{\min}_m$. The same observations hold for $\bel_m$ and $bel^{\min}_m$. Furthermore, the graphs of  $\pl^{k=2}_m$ and $\cp^m$ are quite similar -- up to events with 1 element.

By way of example, Figure \ref{fig:ks3} shows the confidence intervals of $\cp^m$ for 50, respectively, 500 observations.
The organisation of the $x$--axis in Figure \ref{fig:ks3} is the same as in Figure \ref{fig:ks2}. It can be see from Figure \ref{fig:ks3} that -- given a quite sparse PKS as our example of Figure \ref{fig:ks1} -- the 95\% confidence bounds are quite narrow, even if we assume a small empirical basis of only 50 observations (left part of the figure).  An empirical basis of 500 item patterns allows us a precise estimate of the $\cp^m$ values. The same is true for the other measures; we omit the details for these which can be found in the archive.

\section{Summary and outlook}

We have exhibited a common form of several estimators employed in reasoning under uncertainty. The novelty is not that connections exist among them -- these have been known for some time --, but the interpretation as filter functions, a term we have borrowed from digital imaging. A filter, such as an edge detector, extracts salient features of a scene, or, as in our case, of a situation for further processing.  A simulation study indicates how some filters behave in various situations.

In future work we shall explore whether and how the filter concept can be extended to other estimators, for example, to kernel functions such as k -- nearest neighbour. We will also investigate a logical approach to filter functions applied in applications of theories of visual perception and digital imaging, following the path started in \cite{dg_things}.

\section*{Acknowledgement}

We are grateful to the referees for constructive comments.


\begin{thebibliography}{10}
\providecommand{\url}[1]{\texttt{#1}}
\providecommand{\urlprefix}{URL }

\bibitem{dempster1967}
Dempster, A.P.: Upper and lower probabilities induced by a multivalued mapping.
  The Annals of Mathematical Statistics  38(2),  325--339 (1967)

\bibitem{dg_ksencon}
D{\"u}ntsch, I., Gediga, G.: {A} note on the correspondences among entail
  relations, rough set dependencies, and logical consequence. Journal of
  Mathematical Psychology  45,  393--401 (2001), {MR} 1836895

\bibitem{dg_things}
D{\"u}ntsch, I., Gediga, G.: On the gradual evolvement of things. In: Skowron,
  A., Suraj, Z. (eds.) Rough Sets and Intelligent Systems. Professor
  Zdzis{\l}aw Pawlak in Memoriam, vol.~1, chap.~8, pp. 247--257. Springer
  Verlag, Heidelberg (2012)

\bibitem{fh91}
Fagin, R., Halpern, J.: Uncertainty, belief, and probability. Computational
  Intelligence  7(3),  160--173 (1991)

\bibitem{fd11}
Falmagne, J.C., Doignon, J.P.: Learning Spaces. Springer Verlag, Heidelberg
  (2011)

\bibitem{fal90}
Falmagne, J.C., Koppen, M., Villano, M., Doignon, J.P., Johannesen, J.:
  Introduction to knowledge spaces: {H}ow to build, test and search them.
  Psychological Review  97 (1990)

\bibitem{hf92}
Halpern, J.Y., Fagin, R.: Two views of belief: belief as generalized
  probability and belief as evidence. Artificial Intelligence  54,  275--317
  (1992)

\bibitem{mani17}
Mani, A.: Probabilities, dependence and rough membership functions.
  International Journal of Computers and Applications  39(1),  17--35 (2017)

\bibitem{oed_filter}
{Oxford English Dictionaries}: Definition of ``filter''.
  \url{https://en.oxforddictionaries.com/definition/filter} (2018), accessed
  March 20, 2018

\bibitem{paw91}
Pawlak, Z.: {R}ough sets: {T}heoretical aspects of reasoning about data, System
  Theory, Knowledge Engineering and Problem Solving, vol.~9. Klu\-wer,
  Dordrecht (1991)

\bibitem{sha76}
Shafer, G.: A Mathematical Theory of Evidence. Princeton University Press
  (1976)

\bibitem{sko90}
Skowron, A.: {T}he rough sets theory and evidence theory. Fundamenta
  Informaticae  13,  245--262 (1990)

\bibitem{sv95}
S{\l}owi{\'n}ski, R., Vanderpooten, D.: {S}imilarity relations as a basis for
  rough approximations. ICS Research Report~53, Polish Academy of Sciences
  (1995)

\bibitem{smets88a}
Smets, P.: Belief functions. In: Smets, P., Mandani, A., Dubois, D., Prade, H.
  (eds.) Non-standard logics for automated reasoning. Academic Press, London
  (1988)

\bibitem{smets88}
Smets, P.: Belief functions versus probability functions. In: Bouchon, B.,
  Saitta, L., Yager, R.R. (eds.) Uncertainty and Intelligent Systems,
  Proceedings of the 2nd International Conference on Information Processing and
  Management of Uncertainty in Knowledqe-Based Systems IPMU '88. Lecture Notes
  in Computer Science, vol. 313, pp. 17--24 (1988)

\bibitem{sk_tbm}
Smets, P., Kennes, R.: The transferable belief model. Artificial Intelligence
  66(2),  191--234 (1994)

\bibitem{hui03}
Wang, H.: Contextual probability. Journal of Telecommunications and Information
  Technology  3,  92--97 (2003)

\bibitem{wd05}
Wang, H., Dubitzky, W.: A flexible and robust similarity measure based on
  contextual probability. In: Proceedings of the Nineteenth International Joint
  Conference on Artificial Intelligence (IJCAI). pp. 27--34 (2005)

\bibitem{wdgg_noknn}
Wang, H., D{\"u}ntsch, I., Gediga, G., Guo, G.: {N}earest {N}eighbours without
  $k$. In: Dunin-Keplicz, B., Jankowski, A., Skowron, A., Szczuka, M. (eds.)
  {M}onitoring, {S}ecurity, and {R}escue {T}echniques in {M}ultiagent
  {S}ystems, chap.~12, pp. 179--189. Advances in Soft Computing, Springer
  Verlag, Heidelberg (2006)

\bibitem{wm08}
Wang, H., Murtagh, F.: A study of the neighborhood counting similarity. IEEE
  Transactions on Knowledge and Data Engineering  20(4),  449--461 (2008)

\bibitem{yl08}
Yager, R., Liu, L. (eds.): Classic Works of the Dempster-Shafer Theory of
  Belief Functions, Studies in Fuzziness and Soft Computing, vol. 219. Springer
  Verlag, Heidelberg (2008)

\end{thebibliography}

\end{document}